\documentclass[letterpaper, 10 pt, conference]{ieeeconf}  % Comment this line out if you need a4paper

\usepackage[pdftex]{graphicx}        % standard LaTeX graphics tool
\usepackage{mathptmx}
\usepackage{makeidx}
\usepackage{multicol}
\usepackage{footmisc}
\usepackage{textcomp}
\usepackage{algorithm}
\usepackage{algpseudocode}
\usepackage{mathrsfs}
\usepackage{paralist}
\usepackage[font=small,format=plain,labelfont=bf]{caption}
\usepackage{amssymb}
\usepackage{theorem}
\usepackage{footnote}
\usepackage{subfig}
\usepackage{multirow}
\usepackage{array}
\usepackage{float}
\usepackage[usenames,dvipsnames]{color}
\usepackage{bm}
\usepackage{mdwlist}
\usepackage{cite}
\usepackage{amsmath}
\usepackage{gensymb}
\usepackage{xspace}
\usepackage{flushend}

\newcommand{\todo}[1]{
  \textcolor{red}{\footnotesize \textsf{#1}}
}

\newcommand{\figheight}{1in}

\newtheorem{theorem}{Theorem}[section]
\newcommand{\CT}{\ensuremath{{\cal T}^C}\xspace}
\newcommand{\ignore}[1]{}

\IEEEoverridecommandlockouts                              % This command is only needed if 
\title{\LARGE \bf
A Parallel Distributed Strategy for Arraying a Scattered Robot Swarm
}

\author{Dominik Krupke$^{1}$, Michael Hemmer$^{1}$, James McLurkin$^{2}$, Yu Zhou$^{2}$, and S{\'a}ndor P. Fekete$^{1}$%, Yu Zhou$^{2}$% <-this % stops a space
\thanks{$^{1}$S{\'a}ndor P. Fekete, Michael Hemmer, and Dominik Krupke are with the Department of Computer Science at TU Braunschweig, Braunschweig, Germany;
        {\tt\small s.fekete@tu-bs.de, mhsaar@gmail.com, d.krupke@tu-bs.de}}%
\thanks{$^{2}$James McLurkin and Yu Zhou are 
        with the Computer Science Department at
        Rice University, Houston, TX, USA;
        {\tt\small jmclurkin@rice.edu}}%
}

\begin{document}

\maketitle
\thispagestyle{empty}
\pagestyle{empty}

%%%%%%%%%%%%%%%%%%%%%%%%%%%%%%%%%%%%%%%%%%%%%%%%%%%%%%%%%%%%%%%%%%%%%%%%%%%%%%%%
\begin{abstract}
We consider the problem of organizing a scattered group of $n$ robots 
in two-dimensional space, with geometric maximum distance $D$ between robots. 
The communication graph of the swarm is connected, but there is no central authority
for organizing it. We want to arrange them into a sorted and equally-spaced array between
the robots with lowest and highest label, while maintaining a connected
communication network.

In this paper, we describe a distributed method to accomplish these goals,
without using central control,
%\todo{there is no real
%central control but $r_{min}$ still has a central role and induces some
%algorithms!}, 
while also keeping time, travel distance and communication cost
at a minimum.  We proceed in a number of stages (leader election, initial path
construction, subtree contraction, geometric straightening, and distributed
sorting), none of which requires a central authority, but still accomplishes
best possible parallelization. The overall arraying is
performed in $O(n)$ time, $O(n^2)$ individual messages, and $O(nD)$ travel
distance. Implementation of the sorting and navigation use communication
messages of fixed size, and are a practical solution for large populations of
low-cost robots.
\end{abstract}

\section{Introduction and Related Work}
\label{sec:Introduction}

\begin{figure*}[bth]
\vspace*{2mm}
\centering
%\vspace{-1cm}
\includegraphics[width=\textwidth]{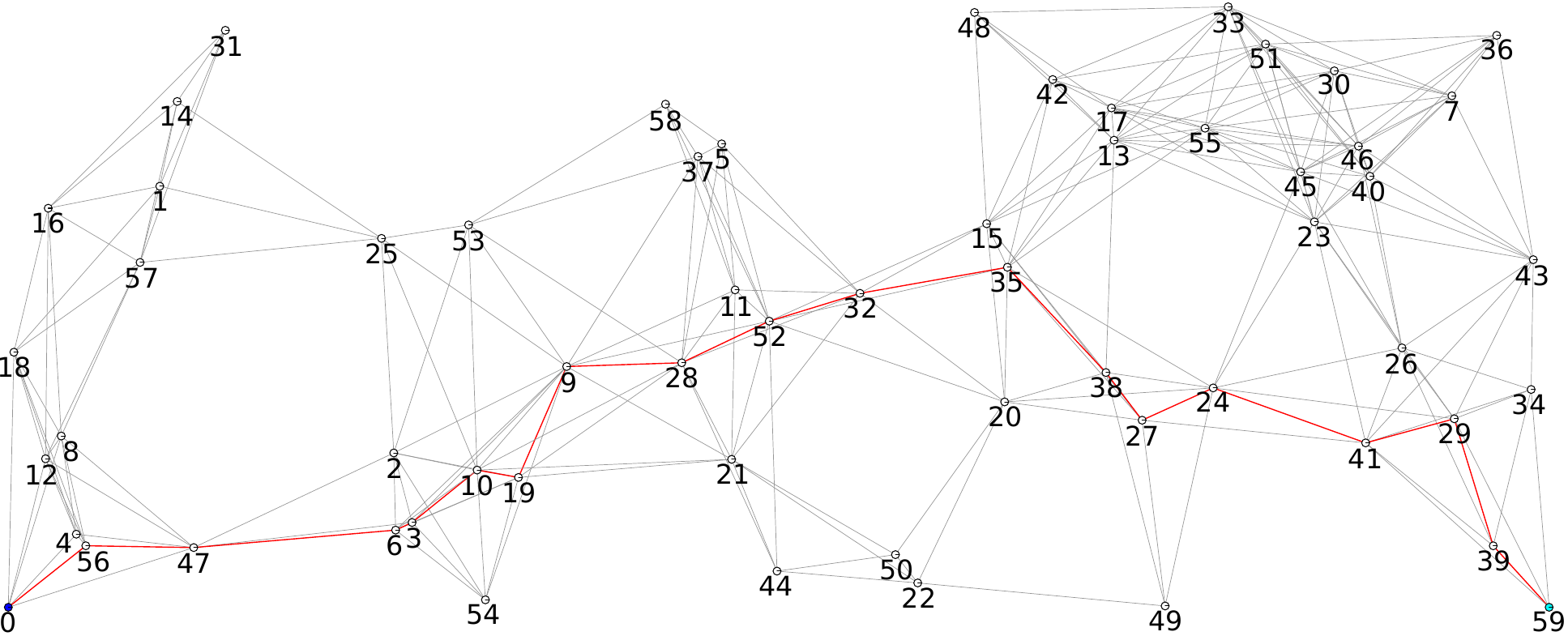}
%\vspace{-2.5cm}
\caption{
\label{fig:arraying-01}
An initial configuration of 60 robots. Edges of the graph $G$ are indicated in light gray. 
The central path (Section~\ref{subsec:initialpath}) from $r_{min}$ to $r_{max}$ is indicated in red. 
}
\end{figure*}

\begin{figure*}[bth]
\centering
%\vspace{-1cm}
\includegraphics[width=\textwidth]{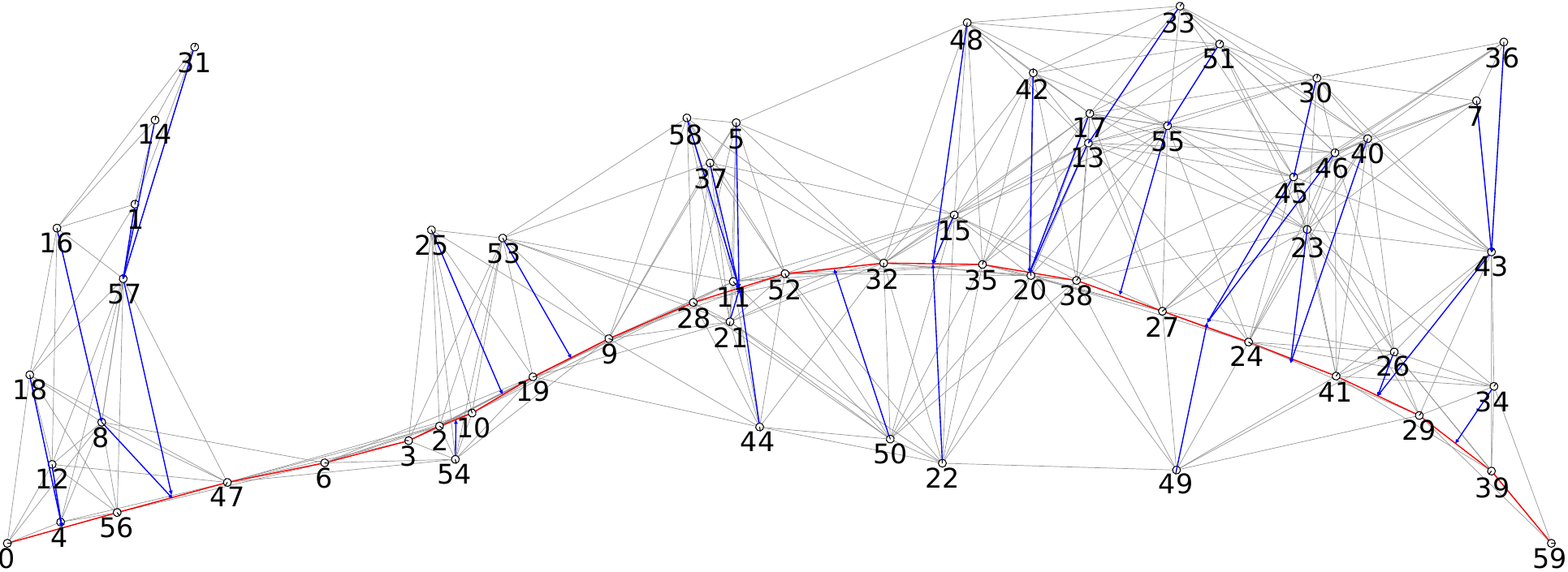}
%\vspace{-2.5cm}
\caption{
\label{fig:arraying-02}
A snapshot of the contraction phase of swarm already depicted in Fig.~\ref{fig:arraying-01}. 
Several robots have already moved onto the central path depicted in red. Robot paths for contraction are depicted in blue. 
It can also be observed that the contraction phase is intervened with the fourth phase as the central path has already straightened significantly.  
}
\end{figure*}

\begin{figure*}[bth]
\centering
%\vspace{-4cm}
\includegraphics[width=\textwidth]{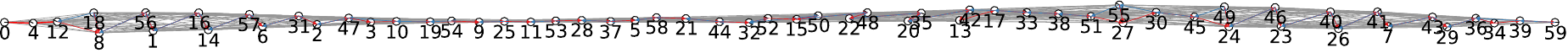}
%\vspace{-2.5cm}
\caption{
\label{fig:arraying-03}
A snapshot of the wave sort phase. After the all robots have integrated into the central path during the contraction phase and as soon as the path is considered straight enough $r_{min}$ initializes sorting waves that propagate through the chain of robots. 
}
\end{figure*}

% something about sorting
%\topic{why sort?}
Consider a large company of $n$ robots after deployment. They are scattered
across a geometric region. While the swarm is still connected in terms of communication, it lacks central
control; and while each of the robots carries a unique ID, none of them
has information about the actual range of labels. How can we get the group
into an organized arrangement: an equally spaced array between the positions of the robots
with minimum and maximum label? 
(See Figure~\ref{fig:arraying} for an example with 30 robots.)
Not only does this demand dealing
with the possibly complicated geometric arrangement in a distributed fashion; 
it also involves sorting them by label, which already requires $\Omega(n\log n)$
in a centralized setting. Furthermore, what are the achievable
time until completion, required communication, and distance traveled?

Arranging robots in a specific order is a necessary routine in some
applications on multi-robot systems. In addition to scenarios after deployment
or perturbation by an uncontrollable event, robots may need to be ordered to go 
through a narrow passage, or to perform sequential procedures. 
For homogeneous robots, swapping tasks can solve this problem in some
applications, but sorting is required if robots have different structures, are
carrying different physical loads that cannot be swapped (and may need to
arrive in order), or differ in some other intrinsic quantities, such as
remaining battery level that cannot be transferred wirelessly~\cite{Litus_FallIn}.

%\topic{the problem}
%\topic{our contribution}
% main figure
\begin{figure}[ht]
\renewcommand{\figheight}{1.6in}
\centering
\includegraphics[width=0.9\linewidth]{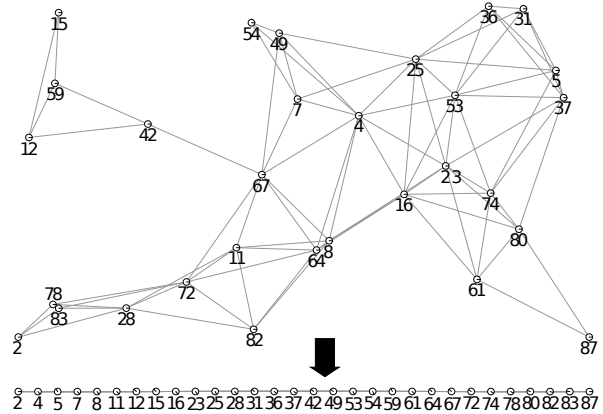}
\caption{
\label{fig:arraying}
An example of arraying a scattered swarm of robots. Thirty robots are initially randomly distributed in space, 
and eventually form an evenly spaced, sorted linear arrangement.
}
\end{figure}

In this paper, we describe an arraying algorithm to arrange robots in a sorted
line. We prove that this algorithm has linear completion time and travel distance,
which are both optimal. The overall protocol is robust and safe for any
arbitrary initial robot configuration in Euclidean space.  We make minimal
assumptions on the robot's communication and
computation requirements, so our approach is suitable for simple robots with
limited resources.  

\subsection*{Related Work}
There has been a considerable amount of work on multi-robot
navigation for labeled robots; e.g., 
see~\cite{SvestkaO98,SanchezL02,BergO05,WagnerC11,WagnerKC12,SolSalHal14,WagnerC15}, 
who demonstrate a number of different challenges and approaches for dealing
with various aspects of the inherent difficulties.
However, there is limited previous work on sorting groups of robots
in a way that combines the requirements of two-dimensional geometry, distributed
computing, and sorting; moreover, the absence of obstacles makes the problem
easier to solve, shifting the focus to achieving optimal efficiency.
The most relevant work is of Litus and Vaughan~\cite{Litus_FallIn}, which uses
a double bracket flow to build a dynamical system to model the robot's
positions.  Running this system will drive the positions to a
sorted ordering.  This is a compact, analytical solution and has provable
properties, but it requires that the robots are initially placed among a line
parallel to some given axis. Additionally, it requires long-range sensing and
communication between robots.

A previous approach to our scenario modeled the system with a discrete model, see
Zhou, Li, and McLurkin~\cite{zhou_physical_2014}. This model does not include
the dynamic properties of the robots: they simply jump to their destination at
each time step. The authors prove convergence with matrix iteration, but also
assume a robot can communicate to its final predecessor and successor.
As it turns out, this method may get stuck if only local communication is available.

Some components of our approach make use of methods for related subproblems.
One of them is the task of straightening a chain of robots in the plane, 
based on purely local methods; this amounts to our problem for the very special
case of a communication graph that is a path, and robots are already sorted by
labels. In a considerable sequence of papers, Meyer auf der Heide et 
al.~\cite{DyniaKLH06,DyniaKHS07,HeideS08,KutylowskiH09,DegenerKKH10,DegenerKH11,DegenerKLHPW11,BrandesDKH11,Cord-Landwehr11,KempkesH11,KempkesH12,KempkesKH12,BrandesDKH13} 
studied versions of the strategy {\sc Go-To-The-Middle} (GTM), in which each
robot moves to the midpoint between its two immediate neighbors. Some of the underlying models
are based on discrete rounds, with robots performing (possibly larger) discrete moves; however,
Degener et al.~\cite{DegenerKKH10} showed that in a setting with
continuous motion and sensing, the variant {\sc Move-On-Bisector} produces
a straight, evenly spaced chain in time bounded by $O(n)$; more recently, Brandes et al.~\cite{BrandesDKH13} 
provided an analysis for continuous GTM
and also established an upper bound of $O(n)$ on the distance traveled by a robot,
and thus, the overall time.
%\todo{More details for convergence speed!}

Another subproblem is actually sorting the robots along the chain. Here the idea
is to exploit parallelism for achieving linear sorting time, thus beating
the lower bound of $\Omega(n\log n)$ for comparison-based sorting.
This has been studied in the context of a stationary array, with synchronized
rounds: a parallel variant of bubblesort called {\em odd-even sort}~\cite{habermann1972parallel}
achieves a runtime of $\Theta(n)$; see Lakshmivarahan, Dhall, and Miller~\cite{LakshmivarahanDM84} 
for a comprehensive analysis.  Note that our distributed scenario does not provide using synchronized rounds, but 
has to work in an asynchronous setting. To this end, we extend the
previous work on odd-even sort to a new, asynchronous variant which we call {\em wave sort}.

In addition to these explicit references, there are also several aspects
which are dealt with implicitly, including local navigation, connectivity,
and collision avoidance. Given the limited amount of space, we refrain from
providing a survey of these related aspects.

%We model our multi-robot system as an embedded graph, and we can separate
%sorting into two components: topological and geometrical.  Topologically,
%sorting involves adding and removing edges from the robot's network to produce
%a sorted path from the lowest robot to the highest robot.  There are many
%analogues to this in the literature, including parallel bubble sort\todo{ref},
%odd/even sort\todo{ref}, \todo{other relevant sorting.  ; parallel sorting for
%a linear set of objects}.  The geometric side of the problem involves moving
%robots to a line, to array them in their sorted order.  We draw inspiration
%from the ``move to the middle'' work of Meyer auf der Heide et al. \todo{ref}
%
%To avoid disconnecting the network among the sorting algorithm, we are also
%inspired by papers related to maintaining connectivity\todo{No, we are simply
%just not losing connectivity}. Williams and Sukhatme\todo{ref} provides a
%potential field model to maintain connectivity with local constrains and
%built-in collision avoidance.

\subsection*{Contribution}
This paper presents a distributed algorithm that arranges a set of robots to a
uniformaly spaced path (implemented as doubly linked list) and sorts them based on some intrinsic
property.  In the end, the robots with minimal and maximal label are 
the beginning and end of this path, with both of them staying put throughout the protocol.
The objective is achieved in total time $\Theta(n)$, and total travel distance $\Theta(nD)$, which are both best possible
assuming small robot diameter compared to the final path length.
The message complexity is $O(n^2)$ after leader election. %\todo{?? We need to discuss message complexity!}
%\todo{($O(n)$ time but possibly exponential individual messages for Chandy-Misra).}
As a subroutine, we propose the distributed $\Theta(n)$ sorting algorithm ``Wave Sort'' 
that is based on Odd-Even-Sort, but works in an asynchronous manner.
%Under assumption of negligible robot radius, the path building algorithm can be
%implemented such that the time and message complexity is dominated by the
%routing tree algorithm.
%\todo{the straightening is not considered for the complexity.}
%\todo{We only have one Experiment WITH collision avoidance that is linear.}
In addition, we give simulation results for an implementation with up to 130 robots,
demonstrating that all components of our overall protocol are indeed linear. %provided
%robot diameters are not too big.
%For a setting with large robot diameters, we also make use of a heuristic for
%collision avoidance. While this is no longer linear in complexity, the additional
%effort for $n\leq 130$ is still marginal compared to the overall runtime.

\section{Model and Assumptions}
\label{sec:ModelAndAssumptions}

We are given a swarm $R$ of $n$ robots $v_i,\ i\in\{1,\ldots,n\}$, where 
each robot has a unique ID number $ID(i)$, the total 
set of these labels is  unknown. 
In addition we are given a total order $\prec$ among them. 
The ID space need not be known in advance; for readability of this article, 
we write $ID(i) = i$ and $v_i \prec v_j$ iff $i<j$.  

%We are given a system of $n$ robots $v_i,\ i\in\{1,\ldots,n\}$. 
%Each robot has an initial position in 2-dimensional space.
%In addition, each robot has a unique ID number $ID(i)$;
%the total set of these labels is unknown, but there is
%a total order among them, i.e, there is a permutation $\pi, \{1,\ldots,n\}\rightarrow\{1,\ldots n\}$,
%such that $ID(\pi(1))<ID(\pi(2))<\ldots<ID(\pi(n))$.

Communication is possible whenever two robots are sufficiently close,
i.e., they are within communication range of each other.
Thus, the robots form the vertices of an undirected
communication graph $G=(V,E)$ with $n$ nodes, %and (initially) $e$ edges, 
in which two vertices
$v_i$ and $v_j$ are connected by an edge in $E$, iff $v_i$ and $v_j$ are close
enough in order to communicate directly without utilizing
other robots. %, i.e., $E$ is the set of all robot-to-robot communication links.
%Note that we do not necessarily assume a unit-disk graph for $G$;
%we only require that communication ranges are star-shaped, i.e., two connected robots
%do not lose connectivity when they move towards each other.
%\todo{Do we?}

The diameter of the initial configuration is the maximum geometric distance
$D:=\max\{Dist_u(v)\mid u,v\in R\}$.
Let robot $v$ be a direct neighbors of robot $u$ in $G$,
then $u$ can measure the relative position of $v$.
%\emph{bearing} $B_u(v)$ and \emph{orientation} 
%$Ori_u(v)$ as depicted in Fig.~\ref{fig:BearingOrientDefinition}, as well as 
%$Dist_u(v)$, which is the distance of the two centers measure by $u$.
We assume that all these measurements are accurate and timely,
unless otherwise noted. Furthermore, we assume that each robot can accurately move 
in the direction of another robot, or towards the midpoint between two visible robots.
We also assume that the time to travel a specific distance $d$ is basically linear in $d$,
as the time to reach maximum speed is small compared to travel time.

%\begin{figure}[tbh]
%\renewcommand{\figheight}{1.6in}
%\centering
%\includegraphics[width=0.3\columnwidth]{figs/bearingori.pdf}
%\caption{
  %\label{fig:BearingOrientDefinition}
  %For neighbor $v$ the robot $u$ can measure the bearing $B_u(v)$
  %as well at the orientation $Ori(v)$ of the neighbor. 
%}
%\end{figure}

%Algorithm execution occurs in a series of synchronous \emph{rounds}, $t_{r}$.
%\todo{Is it? I thought wave sort is based on an asynchronous setting?}
%This greatly simplifies analysis and is straightforward to implement in a
%physical system~\cite{mclurkin_analysis_2008}. At the end of each round, every
%robot $u$ broadcasts a message to all of its neighbors, and also receives a
%message from each neighbor $v \in N(u)$. Each message contains a set of public
%variables, including the sending robot's unique ID number $v.id$. The remaining
%variables will be defined later, but the total message has constant size. Each
%robot may move up to a distance $stepsize$ in each round.

Now the overall task is to achieve a sorted, 
evenly spaced arrangement of robots
between those with lowest and highest ID-number, i.e.,
robot $v_i$ must move to position $p_1+(i-1)\times(p_n-p_1)/(n-1)$, for $i=1,\ldots,n$,
where $p_1$ and $p_n$ are the initial positions of robot $v_1$ and $v_n$, respectively.

%>>Robustness<<
%\subsubsection{Robustness}
%Message losses
%\todo{Michael: following paragraph was moved to the Model part. 
Throughout this paper, we do not explicitly deal with message loss, as the
use of a communication protocol with acknowledgements \cite{Leon-Garcia:2003:CN:861628}
makes our algorithm robust against dropped messages.
The iteration number can be appended to avoid multiplication of messages due to lost acknowledgements.
%As a consequence, each robot will only receive one message of each type per iteration; thus it
%can simply discard all incoming message with a lower iteration number than its
%own.  
A robot will only move if it has received all necessary messages, thus if
a robot starts to move before we received an acknowledgement from it, we know it has received the message.
Overall, delays only slow down execution, but do not harm it.

%The phases contraction and sorting could be merged (if path robots detect
%exterior robots, they switch back to contraction) and then allow the addition
%of non extremal robots.  Non-disconnecting (non-extremal) robot losses lead to
%damages in the doubly linked list and can be fixed by shortening the doubly
%linked list.
%It is also imaginable that idle robots induce new waves to eliminate the problem of single robots delaying a wave.
%However, we did not implement such extensions.

%\todo{Public Variables}

\section{Echo Waves}

%\todo{Text-Description of echo wave algorithm attached}
%\subsection{Echo-Algorithm}
%Echo-Wave algorithm
%We make heavy use of the Echo-Wave algorithm.
%Readers not familiar with it might first read \cite{TODO}.
%General information
The echo algorithm of Chang~\cite{chang1982echo} is a wave algorithm
that is used multiple times in this paper.
This centralized distributed algorithm allows decentralized synchronization in asynchronous settings
and safe broadcasting: if it terminates, the initiator can be sure that everyone is informed.
%The Algorithm
The initiator starts the algorithm by sending a message to all its neighbors.
A robot that receives such a message for the first time broadcasts 
the message to all its neighbors except for the sender of the message, which is saved as predecessor.
Only when robot has also received the message from all its neighbors, 
it sends the message back (echo) to its predecessor.
The algorithm terminates as soon as the initiator has received the message from all its neighbors.
%Properties of the algorithm

One important aspect of this algorithm is that it spans a tree and each robot finishes before its predecessor.
For equal messages delays this tree equals the minimal hop tree towards the initiator.
%Complexity
The algorithm terminates within $O(n)$ time steps and needs $O(n)$ broadcast messages.

\section{Our Algorithm}
\label{sec:Algorithm}

Our algorithm proceeds in a total of five phases. 
The first phase is the {\bf leader election} phase, in which all robots 
identify the extremal ID numbers (Section~\ref{subsec:leader}). 
The second phase then establishes an {\bf initial path} within $G$ between 
the two extremal robots $r_{min}$ and $r_{max}$ (Section~\ref{subsec:initialpath}). 
We then {\bf straighten} (Section~\ref{subsec:straighten}) 
the central path. Integrating the remaing robots into the 
path is achieved by {\bf contracting subtrees} (Section~\ref{subsec:contraction}.)
In the overall protocol (and thus in our simulations), both processes run in parallel; we do present them as two 
different phases for better readability
As soon as all robots are on the central path and as soon as the path 
is straight enough (Section~\ref{subsec:straighten}) we move on to the fifth and last phase,
the {\bf wave sort}, which is a modification of odd-even-Sort, a synchronized, parallel variant of bubble sort. 
%\todo{Odd-Even-Sort}
It requires $O(n)$ time and $O(n^2)$ messages (Section~\ref{subsec:sort}). 
Overall, the first four phases can be seen as building a robot array implemented as 
doubly linked list, while the last phase efficiently sorts the array without deforming it.

%Our algorithm proceeds in a total of five phases, which we summarize as follows, and describe in more detail in the corresponding Sections. 
%\begin{enumerate}
%\item 
%{\em (i)} {\bf Leader Election:} Determine the extremal ID numbers. (Section~\ref{subsec:leader})
%\item 
%{\em (ii)} {\bf Initial Path:} Determine an initial central path between the two extremal robots. (Section~\ref{subsec:initialpath}
%\item 
%{\em (iii)} {\bf Contracting Subtrees:} Move all remaining exterior robots onto the central path. (Section~\ref{subsec:contraction})
%\item 
%{\em (iv)} {\bf Straightening Central Path:} Using local ``move-to-middle'' moves, geometrically 
%straighten the Central Path, and make locations evenly spaced. (Section~\ref{subsec:straighten})
%\item 
%{\em (v)} {\bf Wave Sort:} Use local swaps to sort the robots in the correct order. (Section~\ref{subsec:sort})
%\end{enumerate}
%The first four phases can be seen as building a robot array implemented as doubly linked list, while the last phase sorts such a robot array without deforming the array.

%===========================================================================================================
% PHASE ONE: LEADER ELECTION
%===========================================================================================================
\subsection{Leader Election}
\label{subsec:leader}
% 1. Phase - Leader Election
%
% STATE:
%	The swarm has a connected communication graph but the robots have arbitrary positions and 
%	only know their neighbors for now.
%
% TASK:
%	Determine the robot with the minimal ID and the robot with the maximal ID. It has to be known to all 
%	robots.
%
% METHOD:
%	An algorithm described in "Wan, Distributed Algorithms - An intuitive Approach"
%
% PROBLEMS:
%	When can we be sure that we have found the *global* minimum/maximum (all robots have been checked) -> The
%		echo wave is only able to finish if all robots participated

%Mention origin of algorithm

We determine the robots with smallest and largest ID using a simple leader 
election protocol as described in~\cite{fokkink2013distributed},
based on the echo wave.
Initially every robots claims to be the minimal/maximal robot and starts an echo wave containing its value. 
A robot does not propagate a wave if it already knows of a better value. 
Thus, only $r_{min}$ and $r_{max}$ are able to finish their waves. 
As soon as $r_{max}$ finishes its wave, $r_{max}$ broadcasts a message to inform all robots. %Specifically, this message will reach $r_{min}$. 
As soon as $r_{min}$ has received this message and has finished its echo wave, it initializes another echo wave.
Each robot that finishes this wave can be sure its values are final and transits to the next phase.
As the last robot to finish an echo wave is the initiator, $r_{min}$ is the last robot that enters the next phase.

%Making use of a simple leader election protocol based on the Echo-Wave as described in~\cite{fokkink2013distributed}, we determine the robots with smallest and largest ID.
%In it, every robot that thinks it might be minimal/maximal, starts an Echo-Wave that contains it value.
%Only robots that do not know of a better value participate in this wave and update their known minimal/maximal value.
%Thus, local (and not global) extremes are not able to finish such a wave as the minimal/maximal robot does not participate.
%On the other hand, no robot can know a better value than that of the minimal/maximal robot and thus its wave terminates and overwrites all other minimal/maximal waves.

%Phase transition
%As soon as the maximal robot is determined, it sends a broadcast message to inform the minimal robot (the minimal robot may not yet be determined, so all robots need this information) about it.
%When the minimal robot is determined and has received the broadcast of the maximal robot, it starts a further Echo-Wave.
%Each robot that finishes that wave can be sure its values are final and transits to the next phase.
%As the last robot to finish an Echo-Wave is the initiator, the minimal robot transits last.

%Complexity
The phase has a message complexity of $O(n^2)$ and needs at most $O(n)$ time steps.
%\todo{I think we shouldn't mix different conventions for $|V|$ and $|E|$. Use m instead of e?}
At the end of this phase, no robot has moved, but the robots $r_{min}$ and $r_{max}$ with smallest and largest ID have been identified.

%===========================================================================================================
% PHASE TWO: CENTRAL PATH
%===========================================================================================================
\subsection{Central Path}
\label{subsec:initialpath}
% 2. Phase - Central Path
%
% STATE:
%	All robots are still on the initial positions but Min and Max are known now. Also two (random) trees rooted
%	in Min and Max are known.
%
% TASK:
%	Build a good initial path from the min robot to the max robot into which we can insert the exterior robots
%	in the next phase. The path is represented as a doubly linked list (each robot has two pointers).
%
% METHOD:
%	We let the minimal robot build an shortest path tree (routing tree) using a common algorithm. The initial path is then 
%	extracted by letting the maximal robot sending a message up this tree to the minimal robot. 
%
% Problems:
%	* The path should not intersect itself (knots) -> An shortest path in a unit disc graph with squared Euclidean distance
%			does not intersect itself
%	TODO	

%Intro
The goal of this phase is to establish an initial central path from $r_{min}$ to $r_{max}$ 
within $G$; see Figure~\ref{fig:arraying-01}. 
We could save work by choosing a path that already contains many robots; however, this amounts to
solving the NP-hard problem of a maximum-length path in $G$, which may still not contain all robots.
We sidestep these difficulties by pulling remaining robots to the initial central path, as described further down.
More critical is the property of the path to be free from self-intersections, as these may lead
to irresolvable collisions and blocking during straightening, even for arbitrarily small robot sizes.
We prove that this can be avoided by choosing a shortest path from $r_{min}$ to $r_{max}$ in $G$,
using the squared Euclidean distance as edge weight. 
%In this phase we determine an initial path between $r_{min}$ and $r_{max}$.
%This path does not yet need to contain all robots, but the remaining robots are integrated into this path within the next phase.
%To obtain a favorable path we build a routing tree for $r_{min}$ with the squared Euclidean distance as edge weights and identify the definite path to $r_{max}$ in it.
%Of course we could also use any arbitrary and computational cheaper tree, but we want the initial path to be intersection free and already contain a lot of robots.
%At the end of this phase, all robots know if they are on the initial path and if so its predecessor and successor (doubly linked list).

%Squared Euclidean Distance Path
\begin{theorem}
  The path from the root to a leaf in the 
  routing tree ${\cal T}^R$ of a Unit-Disc-Graph using 
  the squared Euclidean distances as edge weights  
  is intersection free.
\end{theorem}
\begin{proof}
Assume there is an intersection of the path's edges $(a,b)$ and $(c,d)$
with $d(a)<d(b)<d(c)<d(d)$, where $d$ refers to the distance to the root.
Since $(a,b)$ and $(c,d)$ intersect the robots $\{a,b,c,d\}$ form
a convex quadrilateral. 
There is at least one corner of a robot $x\in \{a,b,c,d\}$
with an angle $\geq 90\degree$, which implies that the 
diagonal (either $(a,b)$ or $(c,d)$) formed by its two adjacent 
robots is longer than the two quadrilateral edges at $x$. 
Hence, this diagonal can not be part of the routing tree, a contradiction. 
\end{proof}
Note that this proof in fact only requires edge weights that 
reflect the order induced by their Euclidean length.  

It is not difficult to see that the path remains crossing free throughout the rest of the algorithm.
%see Figure~\ref{fig:contraction}.
%Determine the initial path
%Routing Tree Algorithm
In order to compute the path we first compute a routing tree ${\cal T}^R$ rooted at $r_{min}$
using the Chandy-Misra-Algorithm~\cite{Chandy82distributedcomputation} 
that can be implemented with $O(n)$ time and $O(n\cdot e)$ individual messages 
for a graph with $e$ edges, making use of a synchronizer~\cite{Lakshmanan:1989:EDP:65464.65479}.
Based on broadcast messages (and assuming constant effort for each), the total message complexity is $O(n^2)$.
%\todo{I only flied over the paper and am not 100\% sure that this works for our case}
After ${\cal T}^R$ is established, $r_{max}$ sends a message to $r_{min}$ on ${\cal T}^R$.
Each robot that receives this message forwards it to its parent in ${\cal T}^R$ and becomes part of the initial path.
As soon as the message reaches $r_{min}$, it starts an echo wave on $G$, telling the remaining robots 
that they are not on the path.
Robots which have finished that wave go over to the next phase.
As before, $r_{min}$  is the last robot that enters the next phase.

At the end of this phase, all robots know if they are on the initial path and 
if so its predecessor and successor, that is, the robots on the path emulate a doubly linked list.

%===========================================================================================================
% PHASE Three: STRAIGHTENING THE CENTRAL PATH
%===========================================================================================================
\subsection{Straightening a Path}
\label{subsec:straighten}
% 4. Phase - Straightening the central path
%
% STATE:
%	All robots lie on the central path, which may already be broadly straightened. 
%
% TASK:
% 	Do final straightening and detect if finished. This phase could simply be merged with
%	the previous one. However, it may be better for understanding to keep them separated, as the 
%	previous section is already quite long.
%
%
% METHOD:
%	Move to the middle and TODO 
%
% Problems:
%	TODO

\begin{figure}
\vspace*{2mm}
  \centering
  \includegraphics[width=0.4\textwidth]{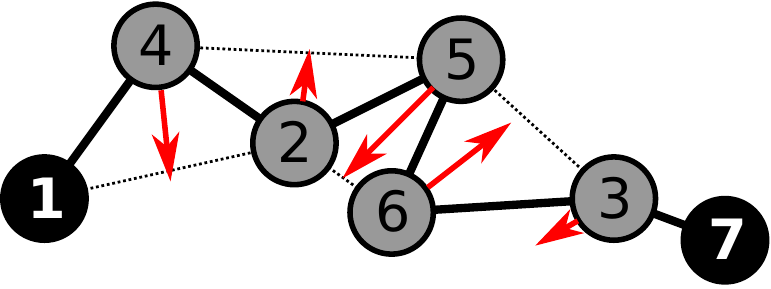}
  \caption{
    Straightening using the continuous {\sc Go-To-The-Middle} (GTM) Method. 
    Not only does it straighten the path, it also makes
    it evenly spaced, as can be seen by $r_3$. 
    In a continuous process the path gets straight.}
  \label{fig:straightening}
\end{figure}

%\todo{Michael: $r_{min}$ and $r_{max}$ don't move at all, what if they are too close?
% Why don't hey just want to move away a little ? Something like $2/3$ distance of 
%the max communication range to the next robot and in direction away from $r_{min}$. }

All robots that are on the central path 
straighten the path by following the continuous GTM method,
for which it is known that the chain of robots converges to 
a straight line (up to inaccuracy of 
measurements)~\cite{DyniaKLH06,DyniaKHS07,HeideS08,KutylowskiH09,DegenerKKH10,DegenerKH11,DegenerKLHPW11,BrandesDKH11,Cord-Landwehr11,KempkesH11,KempkesH12,KempkesKH12,BrandesDKH13}. 
The rule is simple: every robot moves towards the midpoint of the segment 
between its two neighbors $n_l$ and $n_r$ in the doubly linked list;
see Figure~\ref{fig:straightening}.  
In addition, robots on the path are not allowed to cross other path edges 
in order to ensure that the path remains intersection free. 

As mentioned before, this phase only described on its own for 
clearer exposition; in the overall algorithm, it runs parallel to
contraction (Section~\ref{subsec:contraction}), as well as 
sorting (Section~\ref{subsec:contraction}).
During contraction, the process evens out distances among consecutive pairs 
of roots, making space for new incoming robots that are 
integrated into the path. 
During sorting, this process ensures that robots continue 
to have sufficient space. 

%This phase is actually no independent phase but executed by path robots in 
%parallel to contraction as well as sorting phase.
%However, for readability we present it as a normal phase.

%Straightening is done by the move-to-the-middle method, as shown in Figure~\ref{fig:straightening}.
%The positions of the two neighbors $n_l$ and $n_r$ of the doubly linked 
%list are taken and the movement set to their middle point.
%It has been proven that this method converges to \todo{.... TODO.... REF}.
%Of course in practice it is not possible to reach a straight line due to inaccuracy of the sensors.
%Additionally, robots need to take care not intersect with slower robots during the straightening.
%This can be fixed by publishing the two neighbors as public variables and thus the path edges.

%Before starting the sorting, we need to ensure that the robots are straightened enough such that 
%a robot pair has enough space to swap.
%We can extend the termination detection of the straightening process by not only 
%checking if there are exterior robots left but also if each robot has enough space for swapping.
%Alternatively, robot pairs can delay their swap until they have enough space.

%===========================================================================================================
% PHASE FOUR: CONTRACTING SUBTREES
%===========================================================================================================
\subsection{Contracting Subtrees}
\label{subsec:contraction}
This phase integrates remaining robots into the central path; see Figure~\ref{fig:contraction}.
The idea is that robots that are close to the central path 
move towards the midpoint of two consecutive robots and 
get integrated into the path as soon as they have reached 
that position. 
At the same time, these robots are the roots of contracting 
subtrees that organize robots that are further away;
see Figure~\ref{fig:arraying-02}. 
The continued straightening (Section~\ref{subsec:straighten}) 
ensures that there is sufficient space between consecutive robots. 
At the end of this phase, there is a doubly linked list that 
contains all robots with $r_{min}$ and $r_{max}$ at its ends.

%In this phase we integrate the remaining robots into the central path.
%For doing so, the exterior robots move onto the edges of the central path and then get integrated into it.
%However, only those robots that are adjacent to the central path are able to locate an edge of it.
%Thus, we need to build a contraction tree first that allows more distanced robots to find the central path.
%Further, the central path already needs to straighten itself, as too many robots may want to integrate themselves at the same place.
%The straightening process is explained in the next phase.
%The rough process is pictured in Figure~\ref{fig:contraction}.
%At the end of this phase, there is a doubly linked list that contains all robots with minimal and maximal robot at its ends.

\begin{figure}
\vspace*{2mm}
  \centering
  \includegraphics[width=0.4\textwidth]{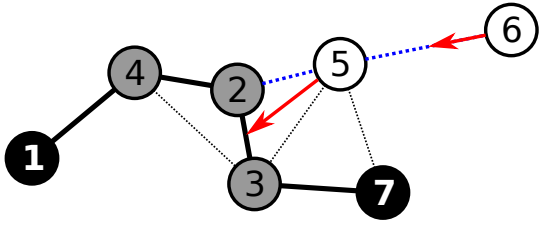}
  \caption{
    Example for contraction. Extremal robots are shown in black, 
    central path robots in grey with thick edges, and contraction tree in doted blue. 
    Robot $r_6$ is not adjacent to the central path and 
    thus moves towards its predecessor. 
    Robot $r_5$ is adjacent to the path and moves to the closest 
    midpoint of a central path edge to get integrated.}
  \label{fig:contraction}
\end{figure}

%1. Step: Contraction Tree
\paragraph{Contraction Tree}
The contraction tree \CT is built by inducing another routing tree rooted at $r_{min}$.
However, edges on the central path get weight zero, while all other edges in $G$ 
get their Euclidean length. 
Thus, we obtain several routing trees that are attached to the central path.  
%Each robot stores its parent as well as its children. 

%2. Step: Echo Wave
\paragraph{Contraction}
The following protocol ensures that robots that are further away move first, 
so parents do not lose connectivity to their children, assuming star-shaped communication ranges.
In order to start a motion, $r_{min}$ induces an echo wave on \CT. 
However, only robots that have finished this wave are allowed to move.
By delaying the finalization of the wave, a child can keep its parent static until it is
ready to move.
In general, children move towards their parent, while 
child robots attached to the central path move 
towards the midpoint of the closest path edge. 
Contractive motion can be parallelized with straightening the path. 

%\todo{Michael: I left the paragraph below unchanged as it is too vague. 
%It is actually not clear what is really done and what could be done.}
%During the contraction, the communication graph may change and shorter paths arise.
%The path distance values can be kept up to date using a simple feedback loop.
%However, switching the parent is a dangerous thing as we are not aligned 
%to the new neighbor and we could lose connection to the new and the old parent while aligning.
%Sometimes it still might be necessary to update the parent, namely if we lost connection to it.
%A tree contraction can lead to an exponential increase of robot density 
%around a parent and robots might get crowded out.
%Then they need to chose a new parent, based on the smallest (updated) path distance value.
%Parents can also wait for endangered children to prevent disconnection.

%4. Step: Integration
\paragraph{Integration}
Each path edge is owned by the incident robot closer to $r_{min}$.
If an exterior robot $r_{ext}$, 
i.e., a robot that is not yet integrated into the central path yet,
is close enough to the midpoint of an edge,
the owner $r_{own}$ of the edge can send an \emph{offer}-message containing 
the ID of itself and the ID of its successor $r_{suc}$ on the central path.
Usually, $r_{ext}$ acknowledges this with an \emph{accept}-message to 
$r_{own}$ and $r_{suc}$, in which case $r_{ext}$ is integrated into the 
central path between $r_{own}$ and $r_{suc}$. 
However, in rare cases $r_{ext}$ may already have received another offer,
in which case $r_{ext}$ replies with a \emph{reject}-message.

%
%\todo{Michael: I left the paragraph below unchanged as it is too vague. 
%It is actually not clear what is really done and what could be done.}
%In general the `add onto edge'-strategy combined with the straightening 
%process will prevent intersections in the central path.
%However, the owner of an edge can also choose wisely to whom it shall 
%offer it (but on the other hand also need to take care that all exterior robots can be integrated).
%For most of the time it can simply add any robot that intersects 
%the edge, as long as the edges are wide enough to give space for a robot.

%5. Step: Termination
\paragraph{Termination}
%\todo{I am not sure that an Echo-Wave is the right choice here. In the implementation I used something like a ping that is echoed by the max robot and only forwarded by robots that are happy. Termination if the min robot gets one back. For a more secure version I would propose a modified one that uses a logical time to check since when it is happy and only forward pings that match this time}
%\todo{I believe that they will believe us that we can determine this even without this long explanation}
The contraction phase terminates as soon as all robots have been integrated into the central path, 
which is checked as follows. 
As soon as $r_{min}$ has no child, it sends an echo wave. 
However, the wave is restricted to the central path and only robots 
without any children are allowed to forward the wave. 
Therefore, all robots are integrated as soon as $r_{max}$ sends the echo. 

In addition, the echo is used to check that the path is straight enough,
i.e., the echo is only sent via path edges that are monotone % \todo{be more precise?}
with respect to the direction from $r_{max}$ to $r_{min}$. 
This is stable even in the presence of accumulated measurement errors and
ensures that robots have sufficient space to move during the next phase. 

%The contraction phase terminates as soon no exterior robot can be seen by any path robot.
%This can be checked by letting $r_{min}$ send a ping on the central path to $r_{max}$, 
%which again sends it back to $r_{min}$.
%Path robots that still see exterior robots deny the forwarding.
%Each ping $r_{min}$ sends, contains an increasing number to build a logical clock.
%Path robots that don't see any exterior robots save the lowest clock value they received from then on.
%If they see an exterior robot again, they delete this value again.
%A ping returning from $r_{max}$ is only forwarded, if the value of this ping is higher than the saved value.
%If a ping comes back to $r_{min}$, there has been a time where no path robot has seen an exterior robot.

%\todo{Michael: '
%  I removed the analysis part. So far we can only 
%  say that experiments are good,
%  but that we do not have a guarantee.  }

%Analysis
%\paragraph{Analysis}
%In the experiments in Section~\ref{TODO} this method is shown to be quite fast but we 
%are not able to give any bounds.
%The straightening process hinders a guarantee on linear complexity, 
%but is necessary to integrate many robots at the same position without 
%letting edges become to small or risking intersections due to displacements.
%If robots are negligible small and collision are thus no problem, 
%the method can be simplified to perform in linear time.
%\todo{This would be a simple tree contraction. Due to the
% Echo-Wave the whole tree can be pulled in at once.}
%However, this is a rather unlikely scenario.

\subsection{Wave Sort}
\label{subsec:sort}
% STATE:
%	The robots are arrayed with the minimal robot at the first position and the maximal 
%	robot at the last position but no further sorting. A robot may only be able to communicate
%	with the previous and the next robot in this array.

% TASK:
%  Sort the array

% METHOD:
%  Distributed Version of parallel Even-Odd-Sorting

% PROBLEMS:
%  1) No global rounds -> Only local rounds that go through the array like a wave. There are 
%			multiple but non-overlapping rounds at a time
%  2) Also the neighbors of a swapping pair may swap -> Before swapping they need to ask them
%			who is going to be on their positions

%>>Algorithm Intro<<
As soon as the robots are arrayed on a sufficiently straight line, %(doubly connected list)
we enter the actual sorting phase.  
The idea is to use a variant of odd-even sort~\cite{habermann1972parallel,LakshmivarahanDM84}, a parallel version of bubble sort
in which alternately robots of odd and even pairs compare their labels and swap their positions if necessary. 
With global control, odd-even sort takes $O(n)$ time by carrying out
at most $n-1$ parallel rounds, each in time $O(1)$.
However, in our distributed setting with no global control, implementing 
global synchronization for each odd (even) round would take $O(n)$ time,
increasing the overall complexity to $O(n^2)$.

Instead, our new \emph{Wave Sort} implements each odd (even) round as a wave that is 
initialized at $r_{min}$. 
If an odd (even) wave reaches and odd (even) pair, they first 
propagate the signal and then start to swap if necessary. 
Thus, the frequency of waves is essentially bounded by the time required to 
swap two consecutive robots which is $O(1)$. 
For large $n$, one can literally see several waves that 
propagate through the chain of robots---see Figure~\ref{fig:arraying-03}.
%Depending on $n$ there are usually several waves that 
%propagate through the chain of robots; see Figure~\ref{fig:arraying-03}.
Overall, the algorithm requires $O(n)$ time,
because the frequency of waves is constant.
In particular, we require at most $n-3$ waves. 

Robots $r_{min}$ and $r_{max}$ are already correctly placed. Consequently,
they will not require swaps; $r_{min}$ initializes waves (Algorithm~\ref{alg:sort:min}),
while $r_{max}$ absorbs waves (Algorithm~\ref{alg:sort:max}). 
Both algorithms are stated for completeness. 
In the following we present Algorithm~\ref{alg:sort:norm} 
for all other robots. % (Algorithm~\ref{alg:sort:norm}). 

\begin{algorithm}%[Min Robot Sorting Routine].
  \begin{algorithmic}[1]
    %\State $\operatorname{iteration}\leftarrow 0$
    \State $m\leftarrow \text{`Master'}$
    \While{not sorted}
    %\State Wait for READY$[i]$ with $i=\operatorname{iteration}$ from $n_r$
    \State Wait for READY$[]$ from $n_r$
    %\If{$\operatorname{iteration} \mod 2 = 1$}
    \State Send $\operatorname{INIT}[m, \operatorname{ID}]$ to $n_r$
    %\Else
    %\State Send $\operatorname{INIT}[\text{`Slave'}, \operatorname{ID}, \operatorname{ID}<n_l]$ to $n_r$
    %\EndIf
    \State Wait for $\operatorname{RET}[r]$
    \State $n_r \leftarrow \min\{n_r,r\}$
    %\State $\operatorname{iteration}++$
    \State $m\leftarrow\overline{m}$
    \EndWhile
    \State Sorted!
  \end{algorithmic}
  \caption{Wave Sort: Min-Robot}
  \label{alg:sort:min}
\end{algorithm}

\begin{algorithm}%[Normal Robot Sorting Routine]
\vspace*{2mm}
  \begin{algorithmic}[1]
    %\State $\operatorname{iteration}\leftarrow 0$
    \While{not sorted}
    \State Wait for $n_l$ in range; Send READY$[]$ to $n_l$%$[\operatorname{iteration}]$
    \State Wait for INIT$[m, l]$ from $n_l$
    \If{$m=\text{`Master'}$}
    \State Send $\operatorname{RET}[\min\{\operatorname{ID},n_r\}]$ to $n_l$
    %\State Wait for READY$[i]$ with $i=\operatorname{iteration}$ from $n_r$
    \State Wait for READY$[]$ from $n_r$
    \State Send $\operatorname{INIT}[\overline{m}, l]$ to $n_r$
    \State Wait for $\operatorname{RET}[r]$ from $n_r$
    \If{$n_r<\operatorname{ID}$}
    \State $n_l \leftarrow n_r;\ n_r \leftarrow r$
    \State Swap with $n_r$, sidestep right
    \Else
    \State $n_l \leftarrow l$
    \EndIf
    \Else 
    %\State Wait for READY$[i]$ with $i=\operatorname{iteration}$ from $n_r$
    \State Wait for READY$[]$ from $n_r$
    \State Send $\operatorname{INIT}[\overline{m}, \max\{\operatorname{ID}, n_l\}]$ to $n_r$
    \State Wait for $\operatorname{RET}[r]$ from $n_r$
    \State Send $\operatorname{RET}[r]$ to $n_l$
    \If{$n_l > \operatorname{ID}$}
    \State $n_r \leftarrow n_l;\ n_l \leftarrow l$
    \State Swap with $n_l$, sidestep right
    \Else
    \State $n_r \leftarrow r$
    \EndIf
    \EndIf
    %\State $\operatorname{iteration}++$
    \EndWhile
  \end{algorithmic}
  \caption{Wave Sort: Non-extremal-Robots}
  \label{alg:sort:norm}
\end{algorithm}

\begin{algorithm}%[Max Robot Sorting Routine].
  \begin{algorithmic}[1]
    %\State $\operatorname{iteration}\leftarrow 0$
    \While{not sorted}
    \State Wait for $n_l$ in range; Send READY$[]$ to $n_l$%$[\operatorname{iteration}]$
    \State Wait for $\operatorname{INIT}[m, l]$ from $n_l$
    \State Send $\operatorname{RET}[\operatorname{ID}]$ to $n_l$
    \State $n_l \leftarrow \max\{n_l,l\}$
    %\State $\operatorname{iteration}++$
    \EndWhile
  \end{algorithmic}
  \caption{Wave Sort: Max-Robot}
  \label{alg:sort:max}
\end{algorithm}

\subsubsection{Algorithm~\ref{alg:sort:norm}}

The message exchange
ensures that all robots know their future left and right neighbors,
so that they can update the pointers of the doubly 
linked list before starting the actual swap (if necessary). 
In general, all incoming messages are stored in a buffer; i.e., 
if a robot \emph{WAITS} for a message, it checks this buffer
until it contains the message (which may already be the case).
On success the message is taken and removed from the buffer.

Consider the bottom of Figure~\ref{fig:sortingpairs}
depicting a master-slave pair and its two neighboring pairs.
The robots of these pairs may also swap. 
In the left pair, the robot that is going to end up at Position~1 
must know $min(r_c,r_d)$, which is going to be its right neighbor. 
Similarly, the right pair must know $max(r_c,r_d)$. At the same time the 
robot ending up at Position~2 must know $max(r_a,r_b)$ and the robot that will 
be at Position~3 must know $min(r_e,r_f)$. After this information is exchanged
the doubly linked list can be updated and (if necessary) 
the robots also change their physical positions 
(Algorithm~\ref{alg:sort:norm}:Line 11 and Line 22). 

%The algorithm only consists of the execution of waves until the robots are sorted.
%These waves are induced by the minimal robot $r_{min}$ and vanish at the maximal robot $r_{max}$.
%Thus, we only need to discuss how a robot executes and forwards a wave as well 
%as the special routine for $r_{min}$ and $r_{max}$.

\subsubsection{Example}
A detailed example is given in Figure~\ref{fig:sortingpairs}.
In the top row, robot~$r_2$ and~$r_3$ are already paired. 
The message exchange of robot~$r_7$ in detail is as follows. 
Robot~$r_7$ starts its loop by sending a \emph{READY}-message 
(A\ref{alg:sort:norm}:L2) to its left neighbor~$r_3$. 
It then waits for a \emph{INIT}-message (A\ref{alg:sort:norm}:L3) from~$r_3$,
which turns~$r_7$ into a master for this round and also indicates 
that the left pair will end the round with~$r_3$ as the robot at Position~1.
$r_7$ then answers with a \emph{RET}-message (A\ref{alg:sort:norm}:L5)
indicating the future right neighbor for~$r_3$, namely~$r_6$. 
%As~$r_2$ and~$r_3$ don't have to swap position,~$r_3$ already 
%finishes its round by setting its pointer to~$r_6$. 
$r_7$ then waits for a \emph{READY}-message (A\ref{alg:sort:norm}:L6) from~$r_6$
and turns it into a slave by an \emph{INIT}-message (A\ref{alg:sort:norm}:L7),
which also contains the future left neighbor of this pair, namely~$r_3$.
It then waits for~$r_6$ to finish its handshake with the right pair and 
finally receives a \emph{RET}-message (A\ref{alg:sort:norm}:L8), which  
contains the future right neighbor for $r_7$, which is~$r_4$.
As master and slave now know their future neighbors,
they update their pointers (A\ref{alg:sort:norm}:L10 and L21) 
and start the actual swap.  

\begin{figure}[tbh]
\vspace*{-2mm}
  \centering
  \includegraphics[width=0.75\columnwidth]{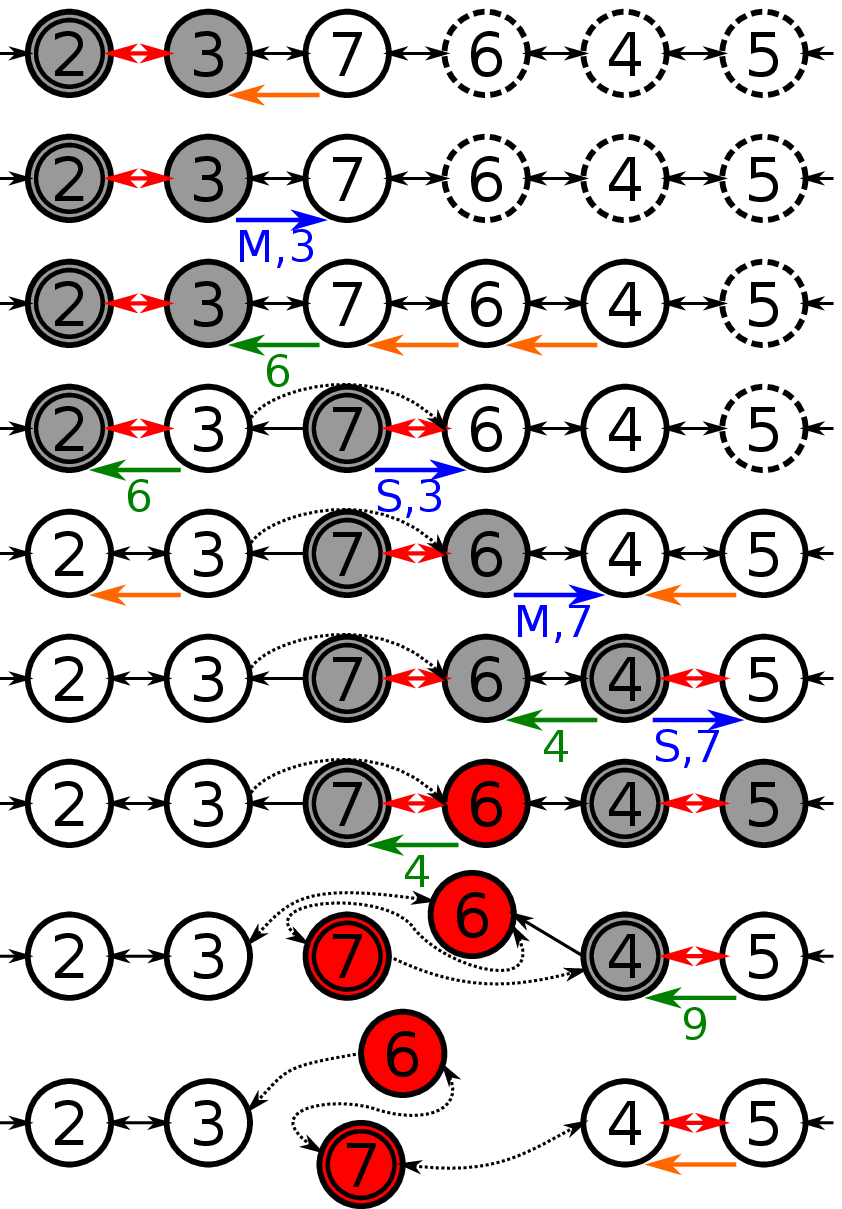}\\
  \vspace{0.5cm}
  \includegraphics[width=\columnwidth]{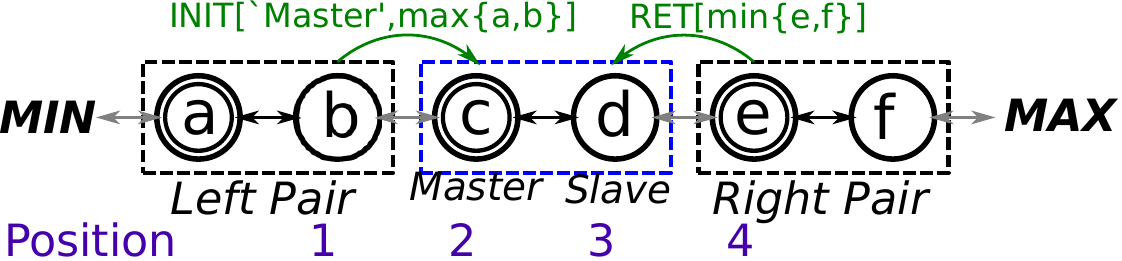}
  \caption{
    Bottom: Sketch of a master-slave pair with its two neighbors. 
    Top: A concrete example of a wave. 
    Robots of central pair $(r_7,r_6)$ swap, while 
    left and right pair are not required to swap. 
    Dotted robots are still in the previous wave. Gray robots are in the current 
    wave. Red robots swap. Dotted pointers are already updated.
    Messages: \emph{READY}---orange; \emph{INIT}---blue; \emph{RET}---green; }
  \label{fig:sortingpairs}
\end{figure}

The important part of this protocol is that the master 
of the next pair always immediately answers 
with a \emph{RET}-message (A\ref{alg:sort:norm}:L8),
as soon as it has received its \emph{INIT}-message (A\ref{alg:sort:norm}:L3).
This \emph{RET}-message already contains the future right neighbor 
of the current pair. 
This information is sufficient for the current pair to update its 
pointers and to start the swap while the wave continues to 
propagate to the right.

\section{Simulation}

We validated our approach by conducting simulations with 
robot swarms of various size. 
As we consider continuous sensing and motion, we assume that 
sensing and motion errors even out.
Robots are simulated as disks, i.e., they may collide. 
The used parameters are inspired by those of the actual r-one robots of Rice University~\cite{rone}:
robot radius is $0.05m$, 
communication range is $4.5m$,
maximal velocity is $1m/s$ and 
maximal acceleration is 1.8$m/s^2$.
A robot can perform around $1.6$ $360\degree$ rotations per second, 
which are necessary to change the direction, as the robots can only drive for- and backwards.
Robots are simulated at  $60hz$ and messages are received within the next time step, i.e., 
after $1/60 s$.

%\todo{This setup is outdated. It produced either many deadlocks or polynomial contraction time with collision avoidance.
%new setup:
The experiments where made for n=15,...,130 robots.
They were randomly placed into a quadliteral of length 0.4*n and height 12.
The minimal robot was placed at the lower left corner and the maximal at the lower right corner;
see Figures~\ref{fig:arraying-01}-\ref{fig:arraying-03} for an overall illustration.
For collision avoidance, we used a heuristic ``retreat from neighbors closer to their goal''.
%There were no failures (disconnections, deadlocks etc.) but some peaks that however were smoothed by the other runs.
Clearly, the running time is linear.
%}
%For $n=15,\ldots,130$, we randomly placed $n$ robots into a square with a diagonal of $n*0.4m$.
%The minimal robot $r_{\min}$ was always placed on the lower left, the maximal robot $r_{\max}$ on the upper right,
%implying that the length of the central path depends linearly on $n$.
%Moreover, there is always enough space to fit all robots onto the central path without 
%moving $r_{min}$ and $r_{max}$, so they never move in this setup. 
%We discarded all initial configuration were the communication graph was either not 
%connected or robot bodies intersected.
%\todo{As for increasing $n$ it is more difficult to find a valid initial configuration we only tested for $n\leq 130$.}

\begin{figure}
\centering
\input{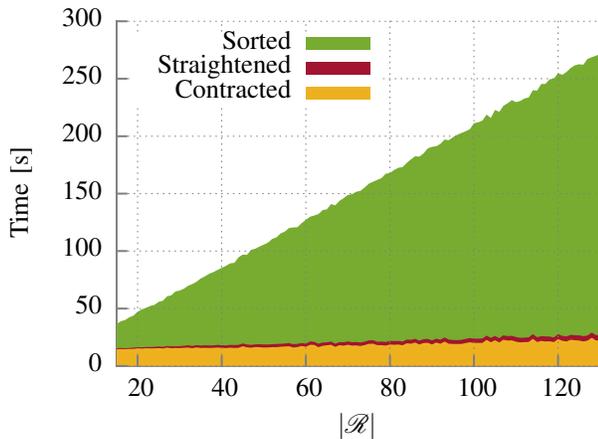}
\caption{Simulation runs for $|\mathcal{R}|=15,\ldots,130$ with 64 runs each; time includes deadlock prevention
during collision-avoiding motion control. 
Clearly, all program components have linear runtime.}
\end{figure}

%\begin{figure}[h]
%\centering
%\includegraphics[width=0.45\textwidth]{./figs/experiments-rev11-58x}
%\caption{\todo{figure without collision avoidance}}
%\label{fig:noavoid}
%\end{figure}

%Note that the square shape increases the difficulty in the contraction phase, as 
%there is higher congestion in the middle part of the path.
%%In some rare cases, this leads to deadlocks; we were able to handle these 
%by a collision avoidance scheme: 
%robots retreat from other robots that are closer to their goal.
%As depicted in Figure~\ref{fig:avoidcollision} this is able to deal with
%scenarios with higher robot density, 
%but at the price that the contraction phase is no longer linear.
%However, the sorting phase, which dominates the runtime for the conducted experiments, 
%remains linear, as argued in Section~\ref{subsec:sort}. 
It is also evident that the time for straightening the chain (the subject of numerous papers dealing with GTM)
is almost negligible, i.e., $O(n)$  with a small constant.

\section{Conclusion}
\label{sec:Conclusion}

We presented a distributed algorithm to sort robots in Euclidean space,
in overall time that is linear time after leader election. 
Robots get sorted from arbitrary initial configuration, even if sensor range is
limited and there is no central control. 
Our underlying assumption is that configurations are dense enough for a
connected communication graph, but not too dense for feasible arrangements
(i.e., available space along the final path is sufficient for accomodating all robots)
or for local collision avoidance.

There are many exciting new challenges that lie ahead. The next step is to combine 
our approach with dense settings in which density along the terminal path
is a problem, and multi-robot collision avoidance comes into play.
We plan to resolve these by moving global minimum and maximum apart
to make the problem feasible. As this generalizes the subproblem of
straightening a chain of robots by a local strategy like continuous {\sc Go-To-The-Middle}
or {\sc Move-On-Bisector},
the mathematical analysis becomes more involved, even if the overall
strategy displays benign behavior toward each other to make problem solvable.

%                                  % on the last page of the document manually. It shortens
%                                  % the textheight of the last page by a suitable amount.
%                                  % This command does not take effect until the next page
%                                  % so it should come on the page before the last. Make
%                                  % sure that you do not shorten the textheight too much.

%\section*{ACKNOWLEDGMENT}

\bibliographystyle{IEEEtran}
\bibliography{lit,mclurkin}

\end{document}